\documentclass[12pt,a4]{elsarticle} %preprint
\usepackage{amsthm}
\usepackage{bbold}
\usepackage{amsthm}
\usepackage{color, soul}
\usepackage[ruled,vlined]{algorithm2e}
\usepackage{array,booktabs}
\usepackage{url}
\usepackage{listings}

\newtheorem{theorem}{Theorem}[section]

\usepackage{graphicx}
\usepackage{amssymb}
\usepackage{amsthm}

\usepackage{lineno}
\usepackage{mathtools}

\theoremstyle{definition}
\newtheorem{definition}{Definition}[section]

\journal{Knowledge Based Systems}

\begin{document}
%\tableofcontents

\begin{frontmatter}


\title{A Model-Agnostic Algorithm for Bayes Error Determination in Binary Classification}

%\title{ILDA: Model Independent Binary Classification Inherent Data Limitation Determination Algorithm}



\author[toelt]{Umberto Michelucci}
\author[polito]{Michela Sperti}
%\author[polito]{Marta Malavolta}
\author[idsia]{Dario Piga}
\author[zhaw,toelt]{Francesca Venturini}
\author[polito]{Marco A. Deriu}

\address[toelt]{TOELT llc, Birchlenstr. 25, 8600 D\"ubendorf, Switzerland}
\address[polito]{PolitoBIOMed Lab, Department of Mechanical and Aerospace Engineering, Politecnico di Torino,
Turin, Italy}
\address[zhaw]{Institute of Applied Mathematics and Physics, Zurich University of Applied Sciences, Technikumstrasse 9, 8401 Winterthur, Switzerland}
\address[idsia]{IDSIA - Dalle Molle Institute for Artificial Intelligence, USI-SUPSI, Via la Santa 1, 6962 Lugano, Switzerland}

\begin{abstract}
This paper presents the intrinsic limit determination algorithm (ILD Algorithm), a novel technique to determine the best possible performance, measured in terms of the AUC (area under the ROC curve) and accuracy, that can be obtained from a specific dataset in a binary classification problem with categorical features {\sl regardless} of the model used. This limit, namely the Bayes error, is completely independent of any model used and describes an intrinsic property of the dataset.
The ILD algorithm thus provides important information regarding the  prediction limits of any binary classification algorithm when applied to the considered dataset.    
In this paper the algorithm is described in detail, its entire mathematical framework is presented and the pseudocode is given to facilitate its implementation. Finally, an example with a real dataset is given.

\end{abstract}


\end{frontmatter}

%\linenumbers

%\section{Datasets}

%TO BE ADDED

% \section{Journal}

% https://ieeexplore.ieee.org/xpl/RecentIssue.jsp?punumber=10207

% https://www.springer.com/journal/10994

% https://www.mdpi.com/journal/algorithms

% CANDIDATE: https://www.journals.elsevier.com/knowledge-based-systems

\section{Introduction}

The majority of  machine learning projects  tend to follow the same pattern. Namely, many different machine learning model types (as decision trees, logistic regression, random forest, neural network, etc.) are first trained from data to predict specific outcomes, and then tested and compared to find the one that gives the best prediction performance on validation data.  Many techniques to compare models have been developed and are commonly used in several settings \cite{raschka2018model, arlot2010survey}. For some specific model types, as neural networks, it is difficult to know when to stop the search \cite{Michelucci2017}. There is always the hope that a different set of hyperparameters, as the number of layers, or a better optimizer will give a better performance \cite{Michelucci2017, yu2020hyper}. This makes the model comparison laborious and time-consuming. 

Many reasons may lead to a bad accuracy: overlapping class densities \cite{garcia2008k, yuan2021novel}, noise affecting the data \cite{schlimmer1986incremental, angluin1988learning, raychev2016learning}, deficiencies in the classifier or limitations in the training data being the most important \cite{tumer2003bayes}. Classifier deficiencies could be addressed by building better models, of course, but other types of errors linked with the data (for example mislabeled patterns or missing relevant features) will lead to an error that cannot be reduced by any optimisation in the model training, regardless of the effort invested. This error is also known in the literature as Bayes error (BE) \cite{Michelucci2017,gareth2013introduction}. The BE can, in theory, be obtained from the Bayes theorem if one would know all density probabilities exactly. However, this is impossible in all real-life scenarios, and thus the BE cannot be computed directly from the data in all non-trivial cases. The Na\"ive Bayes classifier \cite{gareth2013introduction} tries to approximately address this problem, but it is based on the assumption of the conditional independence of the features, rarely satisfied in practice \cite{tumer1998mutual, tumer2003bayes, gareth2013introduction}. The methods to estimate the BE developed in the past decade tend to follow the same strategy: reduce the error linked to the classifier as much as possible, thus being left with only the BE. Ensemble strategies and meta learners \cite{tumer1998mutual, tumer2003bayes, ghosh2002multiclassifier} have been widely used to address this problem. The idea is to exploit the multiple predictions to provide an indication of the limits to the performance for a given dataset \cite{tumer2003bayes}. This approach has been widely used with neural networks, given their universal approximator nature \cite{richard1991neural, shoemakerleast}.

In any supervised learning task, knowing the BE linked to a given dataset would be of extreme importance. Such a value would help practitioners decide whether or not it is worthwhile to spend time and computing resources in improving the developed classifiers or acquiring additional  training data. Even more importantly, knowing the BE would let practitioners assess if the available features in a dataset are useful for a specific goal. Suppose for example that a set of clinical exams are available for a large number of patients. If such a feature set gives a BE of 30\% (so an accuracy of 70\%) in predicting an outcome and a BE smaller than 20\% is the desired target, it is useless to spend time in developing models. So time would be better spent in acquiring additional features.
The problem of determining the BE intrinsic of a given dataset is addressed and solved in this work from a theoretical point of view.

The   contribution  of this paper is twofold.
Firstly, a new algorithm, called {\bf I}ntrinsic {\bf L}imit {\bf D}etermination algorithm (ILD algorithm) is presented. The ILD algorithm allows computing the maximum performance in a binary classification problem, expressed both as the largest area under the ROC curve (AUC) and as the accuracy that can be achieved with any given dataset with categorical features. 
This is by far the largest contribution of this paper, also with respect to previous methods, since the ILD algorithm for the first time allows evaluating the BE for a given dataset exactly. This paper demonstrates how the BE is a limit not dependent on any chosen model but is an inherent property of the dataset itself. Thus, the ILD algorithm gives the upper limit of the prediction possibilities of any possible model when applied to a given dataset, with the only restrictions that the features must be categorical and that the target variable must be binary.
Secondly, the mathematical framework on which the ILD algorithm is based is discussed and a mathematical proof of the algorithm validity is given. The algorithm's computational complexity is also discussed.

The paper is organized as follows. The necessary notation and dataset restructuring for the ILD algorithm is discussed in Section \ref{sec:ILDA}. In Section \ref{sec:ILDA2} the complete mathematical formalism necessary for the ILD algorithm is discussed in detail and the fundamental ILD  Theorem is given and proof is discussed. Application of the ILD algorithm to a real dataset is provided in Section \ref{sec:appl}. Finally, in Section \ref{sec:concl} conclusions and promising research developments are discussed.

\section{Mathematical Notation and Dataset Aggregation}
\label{sec:ILDA}

% In this section we describe the algorithm to determine the best possible ROC Curve that is obtainable, regardless of the model considered. This algorithm, that we call {\bf I}ntrinsic {\bf L}imit {\bf D}etermination {\bf A}lgorithm (ILDA), is based on a Monte-Carlo approach and will allow us to determine the predictive power and limitations on any given categorical datasets.  

%\subsection{Dataset restructuring}

Let us consider a dataset with $N$ categorical features, i.e., each feature can only assume a finite set of values. Let us suppose that the $i^{th}$ feature, denoted as $F_i$, takes $n_i$ possible  values. For notational  simplicity, it is  assumed that the categorical feature is  encoded in such a way that its possible values are  integers from $1$ to $n_i$, with $i=1,...,N$ (note that each $n_i$ can assume different integer values). Each possible combination of the features is called here a {\sl bucket}. The idea is that the observations will be aggregated in buckets depending on their features.
The number of observations present in the dataset are indicated with $M$.  All the observations with the same set of features are said to be in the same bucket. 

The problem is a binary classification one, aiming at predicting an event that can have only two possible outcomes, indicated here with class 0 and class 1. In general, in the $j^{th}$ bucket, there will be a certain number of observations (that we will indicate with $m_0^{[j]}$) with a class of 0, and a certain number of observations (that we will indicate with $m_1^{[j]}$) with a class of 1.

The feature vector for each observation, denoted as $x_k$ (with $k=1,...,M$), is thus defined by an $N$-dimensional vector $(F_{1,k}, F_{2,k}, ..., F_{N,k})  \in \mathbb{N}^N$, where $F_{i,k}$ denotes the value of the $i$-th feature of the $k$-th sample. 

The following useful definitions are now introduced.
\begin{definition}
A {\sl feature bucket} $B^{[j]}$ is a possible combination of the values of the $N$ features, i.e., 
\begin{equation}
    B^{[j]} = \{ (F_1^{[j]}, F_2^{[j]}, ..., F_N^{[j]}) \in \mathbb{N}^N | 
    \, F_i^{[j]} \in \{0,1,...,n_i-1 \}, i = 1,...,N \}
\end{equation}
\end{definition}
In the rest of the paper, the feature bucket is indicated as 
$$B^{[j]} = (F_1^{[j]}, F_2^{[j]}, ..., F_N^{[j]})$$ 
to explicitly mention the feature values characterizing the bucket $B^{[j]}$. %\hfill $\blacksquare$
The total number $N_B$ of feature buckets is thus equal to:   
$N_B = \prod_{i=1}^N n_i$.

As an example, in the case of  two binary features $F_1$ and $F_2$,  four possible feature buckets can be constructed, namely: $B^{[1]}=(0,0)$, $B^{[2]}=(0,1)$, $B^{[3]}=(1,0)$ and $B^{[4]}=(1,1)$. 

%We will indicate with $M^{[i]}$ the set of the indices of the observations in a dataset belonging to a given feature bucket $B^{[i]}$.
\begin{definition}
The set $M^{[j]}$ of the indices of observations belonging to the $j$-th feature bucket $B^{[j]}$ is defined as 
\begin{equation}
    M^{[j]} = \{
    k \in [1,\ M] \, |  \, F_{0,k} = F_0^{[j]} \, \textrm{and} \, F_{1,k} = F_1^{[j]} \, \textrm{and} ... \, F_{N,k} = F_N^{[j]}
    \}
\end{equation}
\end{definition} 
The cardinality of the set $M^{[j]}$ will be denoted as $|M^{[j]}|$.  %Note that, for a given  dataset, not all the feature buckets will be populated.
In a binary classification problem, the observations belonging to the feature bucket $B^{[j]}$, denoted as \begin{equation}
    O^{[j]} = \{
    x_k \, | \, k \in M^{[j]}
    \}
\end{equation}
will contain $m_0^{[j]} \in \mathbb{N}_0$ observations with a target variable equal to $0$ and $m_1^{j]}  \in \mathbb{N}_0$ observations with a target variable equal to $1$. Note that by definition
\begin{equation}
    |M^{[j]}| = m_0^{[j]} + m_1^{[j]}
\end{equation}
and 
\begin{equation}
    M \equiv \sum_{j=1}^{N_B} |M^{[j]}| = \sum_{j=1}^{N_B}  (m_0^{[j]} + m_1^{[j]})
\end{equation}

Based on the definitions above, the original dataset can be equivalently represented by a new dataset of buckets, an {\sl aggregated dataset} $B$,
 each of which contains a certain number of samples $|M^{[j]}|=m_0^{[j]}+m_1^{[j]}$. A  visual representation of the previously described re-arrangement of the original dataset is reported in Figure~\ref{fig:dataset_aggregation} for a dataset with two binary features.
 \begin{figure}
\begin{center}
\includegraphics[width=13cm]{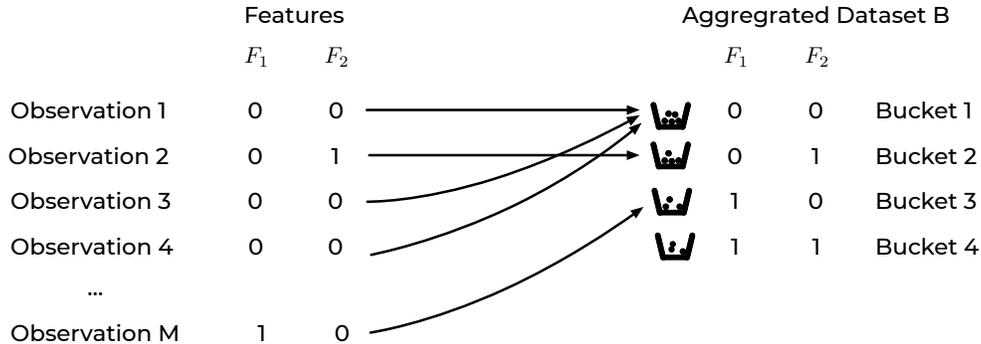}
\caption{An intuitive representation of the dataset aggregation step for a dataset with two binary features $F_1$ and $F_2$. Observations with, for example, $F_1=0$ and $F_2=0$ will be in bucket 1 in the aggregated dataset $B$. Features with $F_1=0$ and $F_2=1$ in bucket 2 and so on.}
\label{fig:dataset_aggregation}
\end{center}
\end{figure}

In the aggregated dataset $B$ each record is thus a feature bucket $B^{[j]}$, characterized by  the number  of observations  of class 0 ($m_0^{[j]}$) and the number of observations of class 1 ($m_1^{[j]}$). In the previous example of a dataset with only two binary features, $B$ would look like the one in Table \ref{tab:aggr_ds}. In this easy example a dataset with any number of observations $M$ would be reduced to one with only 4 records, i.e., the number of possible buckets.
\begin{table}[h!]
\centering
 \begin{tabular}{c ||c |c ||c|c} 
 \hline
 Bucket &  Feature 1 & Feature 2 & Class 0 & Class 1\\ 
 \hline
 1 & $F_1^{[1]}=0$  & $F_2^{[1]}=0$ & $m_0^{[1]}$ & $m_1^{[1]}$ \\
 \hline
 2 & $F_1^{[2]}=0$  & $F_2^{[2]}=1$ & $m_0^{[2]}$ & $m_1^{[2]}$ \\
 \hline
 3 & $F_1^{[3]}=1$  & $F_2^{[3]}=0$ & $m_0^{[3]}$ & $m_1^{[3]}$ \\
 \hline
 4 & $F_1^{[4]}=1$  & $F_2^{[4]}=1$ &$m_0^{[4]}$ & $m_1^{[4]}$ \\ [1ex] 
 \hline
\end{tabular}
\caption{An example of an aggregated dataset $B$ with only two binary features.}
\label{tab:aggr_ds}
\end{table}
 
With this new representation of the dataset, generated by aggregating all observations sharing the same feature values in buckets, the proposed ILD  algorithm allows computing the best possible ROC curve considering {\sl all possible} predictions.

\section{ILD Algorithm Mathematical Framework}
\label{sec:ILDA2}

Since the output of any machine learning model is a function of the feature values, and since a bucket is a collection of observations all with the same feature values, any possible deterministic model will associate to the $j^{th}$ bucket of features only one possible class prediction $p^{[j]}$ that can be either 0 or 1. More in general, to each model can be associated a prediction vector ${\bf p}=(p^{[1]},p^{[2]},...,p^{[N_B]})$. In the next sections important quantities (as TPR and FPR) evaluated for the aggregated dataset $B$ as functions of $m_0^{[i]}$, $m_1^{[i]}$ and ${\bf p}$ are derived.

\subsection{True Positive, True Negative, False Positive, False Negative}

In the feature bucket $i$, if $p^{[i]}=0$ then only $m_0^{[i]}$ observations would be correctly classified. On the other side, if $p^{[i]}=1$ only $m_1^{[i]}$ observations would be correctly classified. For each bucket $i$ the true positive  $TP^{[i]}$ can be written as
\begin{equation}
TP^{[i]} = m_1^{[i]}p^{[i]}
\end{equation}
In fact, if $p^{[i]}=0$, then $TP^{[i]}=0$, and if $p^{[i]}=1$ then $TP^{[i]}=m_1^{[i]}$. Considering the entire dataset, the true positive, true negative ($TN$), false positive ($FP$), false negative ($FN$) are given by:
\begin{eqnarray}
TP &=& \sum_{i=1}^{N_B} TP^{[i]} = \sum_{i=1}^{N_B} m_1^{[i]} p^{[i]}\label{eq:tp}\\
TN &=& \sum_{i=1}^{N_B} TN^{[i]} = \sum_{i=1}^{N_B} m_0^{[i]} (1-p^{[i]})\label{eq:tn}\\
FP &=& \sum_{i=1}^{N_B} FP^{[i]} = \sum_{i=1}^{N_B} m_0^{[i]} p^{[i]}\label{eq:fp}\\
FN &=& \sum_{i=1}^{N_B} FN^{[i]} = \sum_{i=1}^{N_B} m_1^{[i]}(1-p^{[i]})\label{eq:fn}
\end{eqnarray}
where the sums are performed over all the $N_B$ buckets.

\subsection{Accuracy}
In a binary classification problem, the accuracy is defined as
\begin{equation}
a = \frac{TP+TN}{M}.
\end{equation}
Using equations (\ref{eq:tp}) and (\ref{eq:tn}) %derived for the $TP$, $TN$, $FP$ and $FN$, 
the accuracy can be rewritten  as
\begin{equation}
a = \frac{1}{M} \sum_{i=1}^{N_B} \left[
p^{[i]}(m_1^{[i]}-m_0^{[i]})+m_0^{[i]}
\right]
\label{eq:accuracy}
\end{equation}

The maximum value of the accuracy is obtained if the model predicts $p^{[i]}=1$ as soon as $m_1^{[i]}>m_0^{[i]}$. This can be stated as
\begin{theorem}
   The accuracy for an aggregated categorical dataset $B$, expressed as Equation (\ref{eq:accuracy}), is maximised by choosing $p^{[i]}=1$ when $m_1^{[i]}>m_0^{[i]}$ and $p^{[i]}=0$ when $m_1^{[i]}\leq m_0^{[i]}$.
 \label{th:max_accuracy}
\end{theorem}
\begin{proof}
The proof is given by considering each bucket separately. Let's consider a bucket $i$ that has $m_1^{[i]}>m_0^{[i]}$. In this case, there are two possibilities:
\begin{equation}
    \begin{aligned}
     p^{[i]}=1  & \rightarrow &  p^{[i]}(m_1^{[i]}-m_0^{[i]})+m_0^{[i]} = 
                    m_1^{[i]}\\
     p^{[i]}=0  & \rightarrow &  p^{[i]}(m_1^{[i]}-m_0^{[i]})+m_0^{[i]} = 
                    m_0^{[i]}\\
    \end{aligned}
\end{equation}
Therefore, the $i^{th}$ contribution to the accuracy in Equation (\ref{eq:accuracy}) is maximised by choosing $p^{[i]}=1$ for those buckets where $m_1^{[i]}\geq m_0^{[i]}$. With a similar reasoning, the contribution to the accuracy for those buckets where $m_1^{[i]}<m_0^{[i]}$ is maximised by choosing $p^{[i]}=0$. This concludes the proof.
\end{proof}

\subsection{Sensitivity and specificity}

The sensitivity or true positive rate ($TPR$) is the ability to correctly predict the positive cases. Considering the entire dataset, the $TPR$ can be expressed using Equations (\ref{eq:tp}) and (\ref{eq:fn}) as 
\begin{equation}
TPR = \frac{TP}{TP+FN} = \frac{\displaystyle\sum_{i=1}^{N_B} m_1^{[i]} p^{[i]}}{\displaystyle\sum_{i=1}^{N_B} \left[m_1^{[i]} p^{[i]} +m_1^{[i]}(1-p^{[i]})
\right]} = \frac{\displaystyle\sum_{i=1}^{N_B} m_1^{[i]} p ^{[i]}}{\displaystyle\sum_{i=1}^{N_B} m_1^{[i]}}
\label{ROCy}
\end{equation}

%\subsection{Specificity (TNR)}
Analogously, the specificity or true negative rate ($TNR$), which is the ability to correctly reject the negative cases, can be written using Equations (\ref{eq:tn}) and (\ref{eq:fp}) as
\begin{equation}
TNR = \frac{TN}{TN+FP} = \frac{\displaystyle\sum_{i=1}^{N_B} m_0^{[i]} (1-p^{[i]})}{\displaystyle\sum_{i=1}^{N_B} \left[m_0^{[i]} (1-p^{[i]}) +m_0^{[i]}p^{[i]}
\right]} = \frac{\displaystyle\sum_{i=1}^{N_B} m_0^{[i]} (1-p ^{[i]})}{\displaystyle\sum_{i=1}^{N_B} m_0^{[i]}}
\end{equation}

\subsection{ROC Curve}
The receiver operating characteristic (ROC) curve is built by plotting the true positive rate $TPR$ on the $y$-axis, and the false positive rate ($FPR$) on the $x$-axis. For completeness, the $FPR=1-TNR$ is
\begin{equation}
\label{ip}
FPR = 1-TNR = \frac{\displaystyle\sum_{i=1}^{N_B} \left[
m_0^{[i]} - m_0^{[i]}(1-p^{[i]})
\right]
}{\displaystyle\sum_{i=1}^{N_B} m_0^{[i]} } = \frac{\displaystyle\sum_{i=1}^{N_B} m_0^{[i]} p^{[i]}}{\displaystyle\sum_{i=1}^{N_B} m_0^{[i]} }
\end{equation}

\subsection{Perfect Bucket and Perfect Dataset}
Sometimes a bucket may contain only observations that are all in class 0 or 1. Such a bucket is called in this paper {\sl perfect bucket} and is defined as follows.
\begin{definition}
A feature bucket $j$ is called {\sl perfect} if one of the following is satisfied
\begin{equation}
\begin{cases}
    m_0^{[j]} = 0 \\
    m_1^{[j]} > 0
\end{cases}
\end{equation}
or
\begin{equation}
\begin{cases}
    m_0^{[j]} > 0 \\
    m_1^{[j]} = 0
\end{cases}
\end{equation}
\end{definition}
It is also useful to define the set of all perfect buckets $P$.
\begin{definition}
The set of all perfect buckets, indicated with $P$ is defined by
\begin{equation}
    P \equiv \{
    B_j \, \ j = 1,..., N_B \, | (m_0^{[j]} = 0 \, \textrm{and} 
    \, m_1^{[j]} > 0) \, \textrm{or} \, (m_0^{[j]} > 0 \, \textrm{and} 
    \, m_1^{[j]} = 0)
    \}
\end{equation}
\end{definition}
Note that by definition, the set $B/P$ contains only {\sl imperfect} buckets, namely buckets where $m_0^{[j]} > 0$ and  $m_1^{[j]} > 0$.
%Note that a hypothetical model will give to observations in bucket $i$ a prediction $p^{[i]} \in \{ 0,1 \}$. 

%\subsection{The case of a perfect dataset}

An important special case is that of a {\sl perfect dataset}, one in which all buckets are perfect. Indicating with $B$ the set containing all buckets, we have $B=P$. It is easy to see that we can create a prediction vector that will predict all cases perfectly, by simply choosing, for feature bucket $j$
\begin{equation}
    p^{[j]} = 
    \begin{cases}
        0 \, \, \textrm{if} \, \, m_0^{[j]} > 0 \\
        1 \, \, \textrm{if} \, \, m_1^{[j]} > 0 .\\
    \end{cases}
\end{equation}
Remember that all feature buckets are perfect, and in a perfect feature bucket $m_0^{[j]}$ and $m_1^{[j]}$ cannot be greater than zero at the same time.
To summarise our definitions we can define:
\begin{definition}
A dataset $B$ (where $B$ is the dataset containing all feature buckets) is called perfect if $B=P$.
\end{definition}
and
\begin{definition}
A dataset $B$ (where $B$ is the dataset containing all feature buckets) is called imperfect if $P \subset B$.
\end{definition}

\section{The Intrinsic Limit Determination Algorithm}
\label{sec:ILDAalg}

Let us introduce the predictor vector ${\bf p}_S  \equiv (1,1,\ldots, 1)$, for which it clearly holds  
\begin{equation}
    \left\{
        \begin{aligned}
        TPR|_{{\bf p}_S} &=& 1 \\
        FPR|_{{\bf p}_S} &=& 1 \\
        \end{aligned}
    \right.
\end{equation}
$TPR|_{{\bf p}_S}$ and $FPR|_{{\bf p}_S}$ indicate $TPR$ and $FPR$ evaluated for ${\bf p}_S$, respectively.

Let us indicate  as a {\sl flip}  the change of a component of a prediction vector from the value of $1$ to the value of $0$. Any possible prediction vector can thus be obtained by a finite series of flips starting from ${\bf p}_S$, where a flip is done only on components with a value equal to $1$.
Let us denote with ${\bf p}_1$  the prediction vector obtained after the first flip, ${\bf p}_2$ after the second, and so on. After $N_B$ flips, the prediction vector will be ${\bf p}_{N_B}  \equiv (0,0,\ldots, 0)$. The $TPR$ and $FPR$ evaluated for a prediction vector ${\bf p}_i$ (with $i=1,\ldots,N_B$) are indicated as $TPR|_{{\bf p}_i}$ and $FPR|_{{\bf p}_i}$. The set of tuples of points' coordinates is indicated with $\mathcal{P}$:
\begin{equation}
    \mathcal{P} = \{(FPR|_{{\bf p}_S}, TPR|_{{\bf p}_S}),
    \ldots, (FPR|_{{\bf p}_{N_B}}, TPR|_{{\bf p}_{N_B}})\}.
\end{equation}
A curve can be constructed by joining the points contained in $\mathcal{P}$ in \emph{ordered segments}, where ordered segments means that the point $(FPR|_{{\bf p}_S}, {TPR}|_{{\bf p}_S})$ will be connected to $({FPR}|_{{\bf p}_1}, {TPR}|_{{\bf p}_1})$;  $({FPR}|_{{\bf p}_1}, {TPR}|_{{\bf p}_1})$ with $({FPR}|_{{\bf p}_2}, \allowbreak {TPR}|_{{\bf p}_2})$;  and so on. The segment that joins the point 
$({FPR}|_{{\bf p}_S}, {TPR}|_{{\bf p}_S})$ with $({FPR}|_{{\bf p}_1}, \allowbreak{TPR}|_{{\bf p}_1})$ is denoted as $s^{[0]}$; the one that joins the points $({FPR}|_{{\bf p}_1}, \allowbreak{TPR}|_{{\bf p}_1})$ and $({FPR}|_{{\bf p}_2}, {TPR}|_{{\bf p}_2})$ is denoted as $s^{[1]}$, and so on. In Figure \ref{fig:segments} a curve obtained by joining the tuples given by the respective prediction vectors obtained with three flips is visualised to give an intuitive understanding of the process.

\begin{figure}
\begin{center}
\includegraphics[width = 12cm]{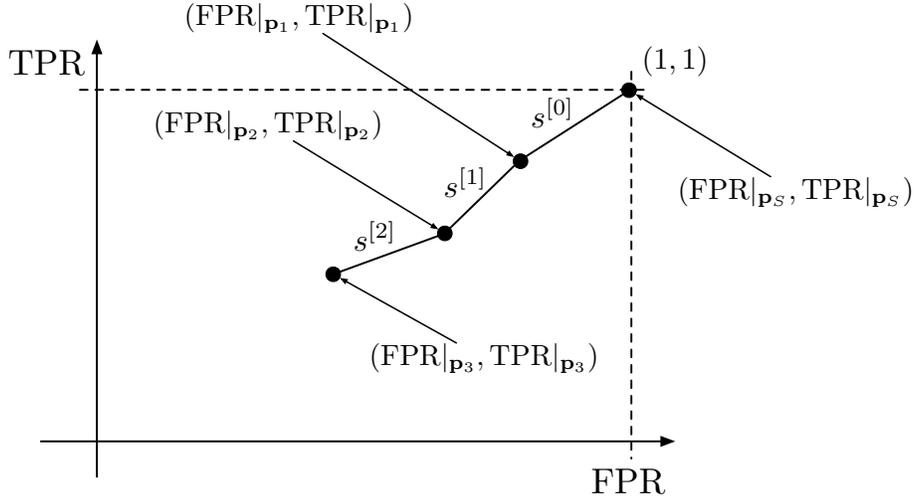}
\caption{Example of a curve  constructed by joining 3 segments obtained after 3 flips.}
\label{fig:segments}
\end{center}
\end{figure}

The ILD algorithm provides a constructive process to select the order of the components to be flipped to obtain the curve $\mathcal{\tilde C}$ characterized by the theoretical maximum AUC that can be obtained from the considered dataset, regardless of the predictive  model used. 

% such that any other prediction vector ${\bf p}$ will gives a tuple $({FPR}|_{{\bf p}}, {TPR}|_{\bf p})$ that lies below $\mathcal{\tilde C}$. In other words, $\mathcal{\tilde C}$ is the curve providing  the theoretical maximum AUC that can be obtained  from the considered dataset, regardless of the predictive  model used. 

To be able to describe the ILD algorithm effectively and prove its validity, some additional concepts are needed and described in the following paragraphs.

\subsection{Effect of one single flip}
Let us consider what happens if one single component, say the $j^{th}$ component, of ${\bf p}_S$ is changed from $1$ to $0$. The $TPR$ and $FPR$ values clearly change. By denoting with ${\bf p}_1$ the prediction vector in which the $j^{th}$ component was changed, the following equations hold:
\begin{equation}
    \left\{
        \begin{aligned}
        {TPR}|_{{\bf p}_1} &=& 1- \frac{m_1^{[j]}}{\displaystyle\sum_{i=1}^{N_B} m_1^{[i]}} \\
        {FPR}|_{{\bf p}_1} &=& 1-\frac{ m_0^{[j]}}{\displaystyle\sum_{i=1}^{N_B} m_0^{[i]}}. \\
        \end{aligned}
    \right.
\end{equation}
%where
%\begin{equation}
%    \left\{
%        \begin{aligned}
%        Q_1 &=& \sum_{i=1}^{N_B} m_1^{[i]} \\
%        Q_0 &=& \sum_{i=1}^{N_B} m_0^{[i]} \\
%        \end{aligned}
%    \right.
%\end{equation}
Therefore, $TPR$ and $FPR$ will be reduced by an amount equal to the ratio $m_1^{[j]}/\sum_{i=1}^{N_B} m_1^{[i]}$ and $m_0^{[j]}/\sum_{i=1}^{N_B} m_0^{[i]}$, respectively. 
As an example, the effect of multiple single flips on $TPR$ and $FPR$ is illustrated in Figure \ref{fig:random_walk}. Here are shown the ROC curves resulting from a random flipping starting from ${\bf p}_S$ for a real-life dataset, namely, the Framingham dataset \citep{mahmood2014framingham} (See Section \ref{sec:appl}). As expected, flipping components randomly results in a curve that lies close to the diagonal. Since the diagonal corresponds to randomly assigning classes to the observations, randomly flipping does not bring to the best prediction possible with the given dataset.
\begin{figure}[hbt]
    \centering
    \includegraphics[scale=0.5]{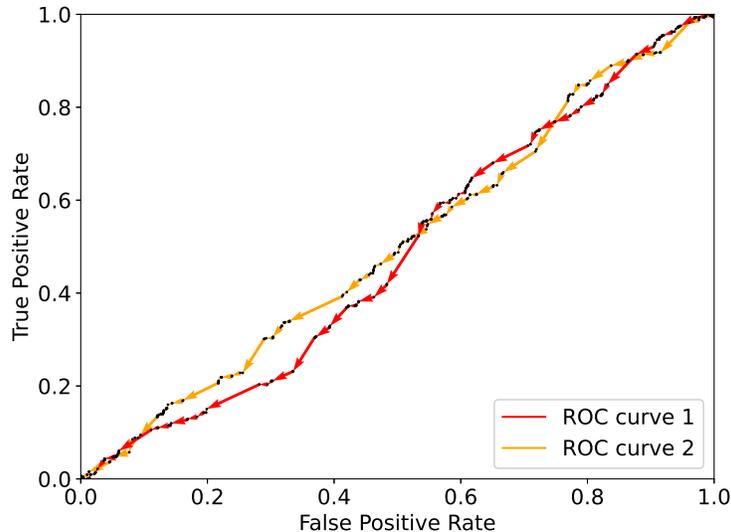}
    \caption{Two examples of ROC curves obtained from random flipping applied to the Framingham dataset \citep{mahmood2014framingham}.}
    \label{fig:random_walk}
\end{figure}

%\hline

By ordering the points in $\mathcal{P}$ in ascending order based on the ratio  ${{TPR}|_{{\bf p}_j}}/\allowbreak{{FPR}|_{{\bf p}_j}}$, a new set of points $\mathcal{\tilde P}$ is  constructed.
%\begin{equation}
%    \mathcal{\tilde P} = (({TPR}_0, {FPR}_0),
%    \ldots, ({TPR}_{N_B}, {FPR}_{N_B}))
%\end{equation}
%where for $\mathcal{\tilde P}$ is valid
%\begin{equation} \label{eqn:orderedratio}
%    \frac{{TPR}_0}{{FPR}_0} \leq \frac{{TPR}_1}{{FPR}_1} \leq \ldots \leq
%    \frac{{TPR}_{N_B}}{{FPR}_{N_B}}
%\end{equation}
It can happen that in a given dataset, multiple points have ${FPR}|_{{\bf p}_j}=0$. In this case, this ratio can not be calculated. If this happens all the points with ${FPR}|_{{\bf p}_j}=0$ can be placed at the end of the list of points. The order between those points is irrelevant. It is interesting to note that a flip for perfect buckets for which $m_0^{[j]}=0$ will have ${{TPR}|_{{\bf p}_j}}=0$ and for all perfect buckets for which $m_1^{[j]}=0$ will have ${FPR}|_{{\bf p}_j}=0$.

With the set of ordered points $\mathcal{\tilde P}$, a curve $\mathcal{\tilde C}$ can be constructed by joining the points in $\mathcal{\tilde P}$ as described in the previous paragraph. Note that the relative order of all points with ${{TPR}|_{{\bf p}_j}}=0$ is also irrelevant, in the sense that this order does not affect the AUC of $\mathcal{\tilde C}$. 

\subsection{ILD Theorem}
The ILD theorem can now be formulated.

\begin{theorem}[ILD Theorem]
Among all possible curves that can be constructed by generating a set of points $\mathcal{P}$ by flipping all components of ${\bf p}_S$ in any order one at a time,  the curve $\mathcal{\tilde C}$ has the maximum AUC.
\end{theorem}

\begin{proof}
The Theorem is proven by giving a construction algorithm. It starts with one set of points $\mathcal P_0$ generated by  flipping components of ${\bf p}_S$ in a random order. Let us consider two adjacent segments generated with  $\mathcal P_0$: $s^{[j]}$ and $s^{[j+1]}$. In Figure \ref{fig:ilda_angles} panel (A) the two segments are plotted in the case where the angle between them $\beta^{[j,j+1]}<\pi$. The angles $\alpha^{[j]}$ and $\alpha^{[j+1]}$ indicate the angles of the segments with the horizontal direction and $\beta^{[j,j+1]}$ the angle between the two segments $j$ and $j+1$. The area under the two segments and any horizontal line that lies below the segments can be increased by simply switching the order of the two flips, as it is depicted in Figure \ref{fig:ilda_angles} panel (B). Switching the order simply means flipping first the $j+1$ component and then the $j$ component.

\begin{figure}[hbt]
\begin{center}
\includegraphics[width =\textwidth]{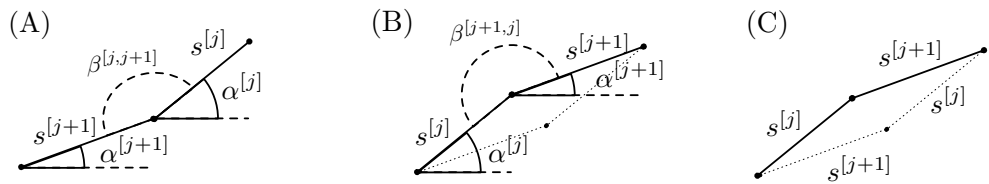}
\caption{Visual explanation for the ILD Theorem. Panel (A): two consecutive segments after flipping components $j$ and then $j+1$; Panel (B): two consecutive segments  after flipping components $j+1$ and then $j$;  Panel (C):  parallelogram representing the difference between the area under the segments in panel (A) and the area under the segments in panel (B).}
\label{fig:ilda_angles}
\end{center}
\end{figure}
It is important to note that in Figure \ref{fig:ilda_angles} panel (A) $\beta^{[j,j+1]}<\pi$ while in panel (B) $\beta^{[j+1,j]}>\pi$. It is evident that the area under the two segments in panel (B) is greater than the one in panel (A). The parallelogram in Figure \ref{fig:ilda_angles} panel (C) depicts the area difference. 

The proof is based on repeating the previous step until all angles $\beta^{[j,j+1]}>\pi$. This is described in pseudo-code in  Algorithm \ref{algo:ilda_proof}.
\begin{algorithm}[h]
\SetAlgoLined
%\KwResult{A dataset $(x_i, y_i)$ for $i=1,...,n$.}
 %initialization\;
 \BlankLine
 %\For{$i=1,...,n$}{
 Generate a set of points $\mathcal P_0$ by  flipping the components of ${\bf p}_S$ in a random sequence\;
 \While{true}{
     $c=0$\;
     \For{$i=1,...,N_B$}{
       \If{$\beta^{[i,i+1]}<\pi$}
       {Switch points $i$ and $(i+1)$ in $\mathcal P_0$\;
       c = c + 1\;}
    }
    \If{$c = 0$} {Exit {\sl While} loop and end the Algorithm}
    The final set of points $\mathcal P_0$ obtained in the loop above will generate the curve $\mathcal {\tilde C}$.
}
 \caption{Algorithm to construct the curve $\mathcal {\tilde C}$.}
 \label{algo:ilda_proof}
\end{algorithm}
The area under the curve obtained with Algorithm \ref{algo:ilda_proof} cannot be made larger with any further switch of points in $\mathcal P$.
% , as doing this would obtain the result of going from panel (B) in Figure \ref{fig:ilda_angles} to panel (A), therefore reducing the area under the switched segments. 
%It is useful to remember that switching points in the set $\mathcal P$ is equivalent to flipping components of the prediction vector ${\bf p}_S$ in a different order.
%QED.
\end{proof}

Note that Algorithm \ref{algo:ilda_proof} will end after a finite number of steps. This can be shown by noting that the described algorithm is nothing else than the bubble sort algorithm \cite{nocedal2006numerical} applied to the set of angles $\alpha^{[1]}$, $\alpha^{[2]}$, ..., $\alpha^{[N_B]}$. So this algorithm has a worst-case and average complexity of ${\mathcal O}(N_B^2)$. 

\subsection{Handling missing values}
%\hl{discutere come vengono gestiti i missing values}

Missing values can be handled by imputing them with a value that does not appear in any feature. All observations that have missing values in a specific feature will be assigned to the same bucket and considered similar, since we have no way of knowing better.

\section{Application of the ILD algorithm to the Framingham Heart Study Dataset}
\label{sec:appl}

The power of the ILD algorithm is best demonstrated by applying it to a real dataset, here the medical dataset named Framingham \citep{mahmood2014framingham}, which is publicly available on the Kaggle website \citep{kaggle}. This dataset  comes from an ongoing cardiovascular risk study made on residents of the town of Framingham (Massachusetts, US). Different cardiovascular risk score versions were developed during the years \citep{wilson1998prediction}, the most current of whom is the 2008 study by \citet{d2008general}, to which the ILD algorithm results are also referred to for comparison of performances. 

The classification goal is to predict, given a series of risk factors, the 10-years risk of a patient of future coronary heart disease. This is a high impact task, since 17.9 million deaths occur worldwide every year due to heart diseases \cite{cardiodiseases} and their early prognosis may be of crucial importance for a correct and successful treatment. The  dataset used in our study contains 4238 patients and 7 features: gender (0: female, 1: male); smoker (0: no, 1: yes); diabetes (0: no, 1: yes); hypertension treatment (0: no, 1: yes); age; total cholesterol; and systolic blood pressure (SBP). The last three features are continuous variables and are discretized as followed: 
\begin{itemize}
\item age: $0\text{ if age}<40\text{ years, }1\text{ if age}\geq40\text{ years and age}<60\text{ years, }\\2\text{ if age}\geq60\text{ years}$;
\item total cholesterol: $0\text{ if total cholesterol}<200~\mathrm{mg/dL}\text{, }\\1\text{ if total cholesterol}\geq200~\mathrm{mg/dL}\text{ and total cholesterol}<240~\mathrm{mg/dL}\text{, }\\2\text{ if total cholesterol}\geq240~\mathrm{mg/dL}$;
\item SBP: $0\text{ if SBP}<120~\mathrm{mmHg}\text{, }1\text{ if SBP}\geq120~\mathrm{mmHg}$. 
\end{itemize}
The outcome variable is binary (0: no coronary disease, 1: coronary disease). 

To correctly interpret the comparison with existing results, it is important to remark that in the dataset used in our study, the high-density lipoprotein (HDL) cholesterol variable is missing with respect to the original Framingham dataset employed in \citep{d2008general}.
Finally, to create the buckets all missing values are substituted by a feature value of -1.

The application of the algorithm starts with the population of the buckets as described in the previous sections. A total number of 177 buckets are generated. The dataset is split into two parts: a training set $S_T$ (80\% of the data) and a validation one $S_V$ (20\% of the data). For comparison, the Na\"ive Bayes classifier is also trained and validated on $S_T$ and $S_V$.
The performance of the ILD algorithm and the Na\"ive Bayes classifier are shown through the ROC curve in Figure \ref{fig:roc_fram}. The AUC for the ILD algorithm is 0.78, clearly higher than that for the Na\"ive Bayes classifier, namely, 0.68.

\begin{figure}
\begin{center}
\includegraphics[scale=0.5]{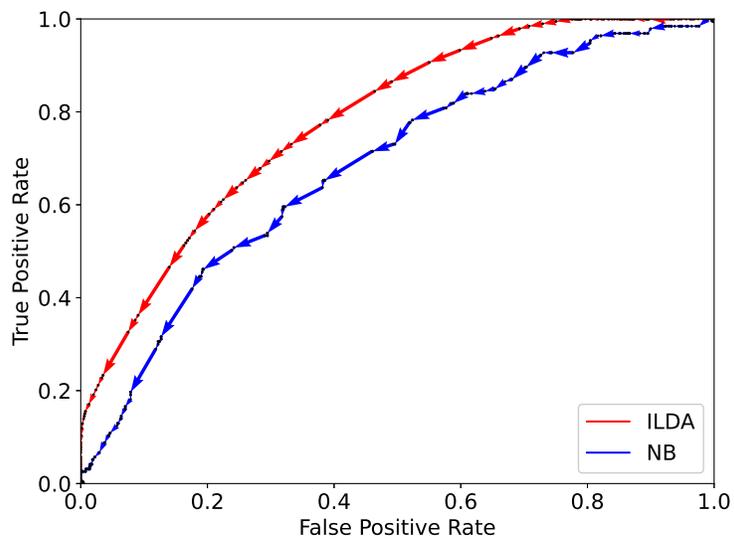}
\caption{Comparison of the performance of the ILD algorithm (ILDA, red) and Na\"ive Bayes classifier (NB, blue) implemented on categorical features based on one single training and validation split.}
\label{fig:roc_fram}
\end{center}
\end{figure}
 
To further test the performance, the split and training is repeated for 100 different dataset splits. Each time both the ILD algorithm and the Na\"ive Bayes classifier are applied to the validation set $S_V$ and the resulting AUCs are plotted in Figure \ref{fig:ilda_fram} (top panel). For clarity, the difference between the AUC provided by the two algorithms is shown in Figure \ref{fig:ilda_fram} (bottom panel).
\begin{figure}
\begin{center}
\includegraphics[scale=0.5]{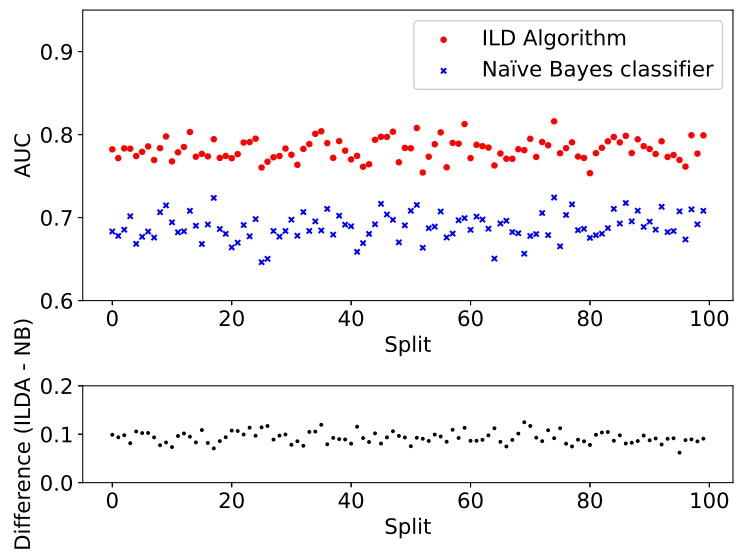}
\caption{Comparison between the performance of the ILD algorithm (red) and Na\"ive Bayes classifier (blue) implemented on categorical features based on 100 different training and validation splits. Top panel: AUC; bottom panel: difference between the AUC provided by the ILD algorithm and the Na\"ive Bayes classifier.}
\label{fig:ilda_fram}
\end{center}
\end{figure}

This example shows that the application of the ILD algorithm allows the comparison of the prediction performance of a model, here the Na\"ive Bayes classifier, with the maximum obtainable for a given dataset. The maximum accuracy over the validation set $S_V$ is 85\% for the Na\"ive Bayes classifier (calculated for a specific threshold, i.e., 61\%, which optimizes the accuracy over the training set $S_T$) and 86\% for the ILD algorithm (calculated applying Theorem \ref{th:max_accuracy}). The reported accuracies are the ones that refer to the ROC curves shown in Figure \ref{fig:roc_fram}. The two values are similar and have been reported for completeness, even if this result may by misleading. The reason lies in the strong dataset unbalance, since only 15\% of the patients experienced a cardiovascular event. In particular, both the Na\"ive Bayes classifier and ILD algorithm obtains a true positive rate and a false positive rate near zero in the point of the ROC curve which maximises the accuracy over the validation set, with a high miss-classification in positive patients, which however cannot be noticed from the accuracy result. As it is known, the accuracy is not a good metric for unbalanced datasets, and the AUC is a much widely used metric that does not suffer from the problem described above.

\section{Conclusions}
\label{sec:concl}

The work presents a new algorithm, the ILD algorithm, which determines, the best possible ROC curve that can be obtained from a dataset with categorical features and binary outcome, regardless of the predictive  model.

The ILD algorithm is of fundamental importance to practitioners because it allows:
\begin{itemize}
    \item to determine the prediction power (namely, the BE) of a specific set of categorical features;
    \item to decide when to stop searching for better models;
    \item to decide if it is necessary to enrich the dataset.
\end{itemize}
The ILD algorithm has thus the potential to   revolutionize how binary prediction problems will be solved in the future, allowing  practitioners to save an enormous amount of efforts,  time, and money (considering that, for example, computing time is expensive especially in cloud environments).

The major limitations of the ILD algorithm are firstly the requirement for the features to be categorical. The generalization of this approach to continuous features is the natural next step and will open new ways of understanding datasets with continuous features. Secondly, the ILD algorithm works well when the different buckets are populated with enough observations. The ILD algorithm would not give any useful information on a dataset with just one observation in each bucket (since it would be a perfect dataset). Consider the example of gray levels images. Even if pixel values could be considered categorical (the gray level of a pixel is an integer that can assume values from 0 to 255), two major problems would arise if the ILD algorithm would be applied to such a case: the number of buckets would be extremely large and each bucket would contain only one image therefore making the ILD algorithm completely useless, as only perfect buckets will be constructed. %In fact if you have an incredibly small $10\times10$ pixel image in gray levels, you would have possible $255^{100}$ buckets (a number with 241 digits). Having two images in a bucket would mean that the two images are identical, since they would have the exact same pixel gray values. This, for example, is a case where ILDA is not useful as it is described here.

An important further research direction is the expansion of the ILD algorithm to detect the best performing models that do {\bf not} overfit the data. In the example of images, it is clear that being a perfect dataset one could theoretically construct a perfect predictor, therefore giving a maximum accuracy of 1. 
The interesting question is how to determine the maximum accuracy or the best AUC only in cases in which no overfitting is occurring. This is a nontrivial problem that is currently under investigation by the authors. To address, at least partially, this problem, the authors have defined a perfection index (IP) that can help in this regard. IP is discussed in Appendix \ref{app:A}.

To conclude, although more research is needed to generalize the ILD algorithm, but it is, to the best knowledge of the authors, the first algorithm that is able to determine the exact BE from a generic dataset with categorical features, regardless of the predictive models.

%\hl{Genetic algorithms as promising development of research to select the optimal something}
 
\appendix 
\section{Perfection Index $I_P$}
\label{app:A}

It is very useful to give a measure of {\sl how} perfect a dataset is with a single number. To achieve this, a perfection index (PI) $I_P$ can be defined.
\begin{definition}
The Perfection Index $I_P$ is defined as:
\begin{equation}
\label{perfindex}
    I_P \equiv \frac{1}{M} \displaystyle \sum_{i=1}^{N_B} |m_1^{[i]}-m_0^{[i]}|.
\end{equation}
\end{definition}
%or
%\begin{equation}
%    \tilde I_P \equiv \displaystyle  \frac{1}{M} \sum_{i=1}^{N_B} \frac{ |m_1^{[i]}-m_0^{[i]}|}{ %m_1^{[i]}+m_0^{[i]}}
%\end{equation}

Note that if a bucket $i$ is perfect then either $m_1^{[i]}$ or $m_0^{[i]}$ are zero. 
In a perfect dataset $B/P = \varnothing$. In this case, from Eq. (\ref{perfindex}) it is easy to see that $I_P=1$ since
\begin{equation}
    I_P = \frac{\displaystyle \sum_{i \in B_1} m_1^{[i]} + \sum_{i \in B_0} m_0^{[i]} }{M} = 1
\label{eq:ip}
\end{equation}
where
\begin{equation}
    B_1 \equiv \{
    j \in [1,N_B]  | m_0^{[j]} = 0
    \}
\end{equation}
and
\begin{equation}
    B_0 \equiv \{
    j \in [1,N_B]  | m_1^{[j]} = 0
    \}
\end{equation}
For an imperfect dataset, it is to see that $I_P$ will be less than 1. In fact
\begin{equation}
    I_P = \frac{\displaystyle \sum_{i \in B_1} m_1^{[i]} + \sum_{i \in B_0} m_0^{[i]} + \sum_{i \in B_{01}} |m_1^{[i]}-m_0^{[i]}|}{M} 
\end{equation}
where
\begin{equation}
    B_{01} \equiv \{
    j \in [1,N_B]  | m_1^{[j]} >0 \textrm{ and } m_0^{[j]} >0
    \}
\end{equation}
that cannot be one as long as $B_{01}$ is not empty. There is a special interesting case when the dataset $B$ is completely imperfect, meaning $B_0 = B_1 = \varnothing$, and for every bucket $i$ with $i=1,...,N_B$ is true that $ m_1^{[i]} =  m_0^{[i]}$. In this case, regardless of the predictions a model may make, the accuracy will always be 0.5. In this case, $I_P=0$. In facts, one can see that in this case
\begin{equation}
    a = \frac{1}{M} \sum_{i\in \tilde B} [p^{[i]}( m_1^{[i]}- m_0^{[i]} ) +  m_0^{[i]}] = \frac{1}{2}
\end{equation}
since $m_1^{[i]} = m_0^{[i]}$ and that 
\begin{equation}
    \sum_{i\in \tilde B}  m_0^{[i]} = \sum_{i\in \tilde B}  m_1^{[i]} = \frac{M}{2}
\end{equation}

This index is particularly useful since the following theorem can be proved.
\begin{theorem} The perfection index satisfy the relationship
\begin{equation}
    I_P = \max_{S_P} a - \min_{S_P} a
\end{equation}
where $S_P$ is the set of all possible prediction vectors for the bucket set $B$.
The perfection index measures that the ranges of possible values that the accuracy ($a$) can have. 
\end{theorem}
\begin{proof}
To start the proof the formula for the maximum ($\max_{S_P} a$) and minimum ($\min_{S_P} a$) possible accuracy must be derived. Starting from Equation (\ref{eq:accuracy}) and choosing (to get the maximum accuracy) $p^{[i]}=1$ for $m_1^{[i]} \geq m_0^{[i]}$ and $p^{[i]}=0$ for $m_1^{[i]} \leq m_0^{[i]}$ the result is
\begin{equation}
    \max_{S_P} a = \frac{1}{M} \left[
        \sum_{m_1\geq m_0}  m_1^{[i]} +
        \sum_{m_1\leq m_0}  m_0^{[i]}
    \right]
\label{eq:amax}
\end{equation}
at the same time, by choosing $p^{[i]}=0$ for $m_1^{[i]} \geq m_0^{[i]}$ and $p^{[i]}=1$ for $m_1^{[i]} \leq m_0^{[i]}$ $\min_{S_P} a$ can be written as
\begin{equation}
    \min_{S_P} a = \frac{1}{M} \left[
        \sum_{m_1\geq m_0}  m_0^{[i]} +
        \sum_{m_1\leq m_0}  m_1^{[i]}
    \right]
\label{eq:amin}
\end{equation}
To prove the theorem, let us rewrite the maximum accuracy from Equation (\ref{eq:amax}) as
\begin{equation}
    \max_{S_P} a = \frac{1}{M} \sum_{m_1\geq m_0} [( m_1^{[i]}- m_0^{[i]} ) +  \sum_{m_1\leq m_0} [( m_1^{[i]}- m_0^{[i]} ) +
    \sum_{m_1\geq m_0}  m_0^{[i]} +
    \sum_{m_1\leq m_0}  m_1^{[i]}
    ] 
\end{equation}
and using Equations (\ref{eq:ip}) and (\ref{eq:amin}) previous equation can be re-written as
\begin{equation}
    \max_{S_P} a = I_P + \min_{S_P} a.
\label{eq:aip}
\end{equation}
This concludes the proof.
\end{proof}

\subsection{Interpretation of the perfection index $I_P$}

In the two extreme cases, if $I_P$ is small then the ILD algorithm will give very useful information. The dataset is "imperfect" enough that an analysis as described is very useful. The more the value of $I_P$ gets close to one, the more the analysis, as it is formulated here, is less helpful. Since in the case of a perfect dataset a perfect prediction model can be easily built and therefore the question of what is the best possible model loses significance. Note that there is a big assumption made, namely, that a feature bucket with a few observations contains the same information as one with one thousand observations in it. In real life, one feature bucket with just one (or few) observation will probably be due to the lack of observations collected with the given set of features and therefore should be doubted in its importance.

%\section{Application to the Framingham Dataset}

%\subsection{Contributions of the Method}

%\subsection{Example of application on Framingham dataste}

\clearpage
\appendix

%\appendixpage



%% References



%% New version of the num-names style
\bibliographystyle{elsarticle-num-names}
\bibliography{bibtex.bib}






% \end{thebibliography}

\end{document}